\documentclass[twocolumn]{IEEEtran}
\usepackage{amsmath,amscd,amssymb,amsgen,amsfonts,amsbsy,amsthm}
\usepackage{mathrsfs}
\usepackage[utf8x]{inputenc}
\usepackage{color}
\usepackage{textcomp}
\usepackage{float}
\usepackage{latexsym,graphicx}

\usepackage{tikz}
\usetikzlibrary{automata,arrows,positioning,calc}
\usepackage{url}
\usepackage{cite}
\usepackage{algorithmic}
\usepackage[ruled,lined]{algorithm2e}

\newcommand{\prob}[1]{\mathsf{Pr}\left( #1 \right)}

\newcommand{\remove}[1]{}

\newcommand{\comments}[1]{}

\newcommand{\given}{\; \big\vert \;} 
\newcommand{\bydef}{:=}
\newcommand{\ip}[2]{\langle #1,#2 \rangle}

\newtheorem{lemma}{Lemma}

\newtheorem{theorem}{Theorem}
\newtheorem{remark}{Remark}
\newtheorem{assumption}{Assumption}

\makeatother

\title{Optimal Recommendation to Users that React:\\Online Learning
  for a Class of POMDPs}

\author{Rahul Meshram, Aditya Gopalan and D. Manjunath
  \thanks{Rahul Meshram and D.~Manjunath are with the Electrical
    Engineering Department of IIT Bombay in Mumbai INDIA. Aditya
    Gopalan is with the Electrical Communication Engineering
    Department of Indian Institute of Science in Bangalore INDIA. The
    work of Rahul Meshram and D. Manjunath was carried out in the
    Bharti Centre for Communications at IIT Bombay. D.~Manjunath is
    also supported by grants from CEFIPRA and DST.} }

\begin{document}

\maketitle
\begin{abstract}
  We describe and study a model for an Automated Online Recommendation
  System (AORS) in which a user's preferences can be time-dependent
  and can also depend on the history of past recommendations and
  play-outs. The three key features of the model that makes it more
  realistic compared to existing models for recommendation systems are
  (1)~user preference is inherently latent, (2)~current
  recommendations can affect future preferences, and (3)~it allows for
  the development of learning algorithms with provable performance
  guarantees. The problem is cast as an average-cost restless
  multi-armed bandit for a given user, with an independent partially
  observable Markov decision process (POMDP) for each item of
  content. We analyze the POMDP for a single arm, describe its
  structural properties, and characterize its optimal policy. We then
  develop a Thompson sampling-based online reinforcement learning
  algorithm to learn the parameters of the model and optimize utility
  from the binary responses of the users to continuous
  recommendations. We then analyze the performance of the learning
  algorithm and characterize the regret. Illustrative numerical
  results and directions for extension to the restless hidden Markov
  multi-armed bandit problem are also presented.
\end{abstract}

\section{Introduction}
\label{sec:intro}

Automated online recommendation (AOR) systems for different types of
content aim to adapt to user's preferences and issue targeted
recommendations for content for which the user is estimated to have a
higher preference. In the generation of these recommendations, user
behavior is typically modeled as a {\em fixed}, stochastic response
governed by the preference, or taste, of the user for each specific
item that has been recommended for consumption. However, in most AOR
systems, the {\em dynamic} aspects of the response of the user to
recommended content are not modeled or investigated. For example,
consider an AOR for music. It is reasonable to assume that for some
items, there will be short-term user fatigue for a song that has been
just been recommended and played out. In this case, the user's
preference for the item drops sharply immediately after consumption
and rises with time subsequently. Hence, for such items it is ideal to
allow for an interval of time before recommending the item again. It
is also possible that the opposite is true---the user's appetite is
whetted with each successive recommendation of a particular item.
Capturing such response dynamics could involve assigning a notion of
state to the user's preference for each item at the time of the
choosing the recommendations. This state should in turn depend on the
history of the recommendations or play-outs for the item.

There are several technical challenges in capturing or estimating the
time dependent user preference to an item. Firstly, even for a
recommended item, the user's taste or preference for an item is never
directly observed; only a binary response in the form of like/dislike,
or play/skip, depending on the preference at that time is
available. Thus the actual preference is a latent quantity which needs
to be inferred and tracked continuously.  This motivates the use of a
hidden Markov model for the state of a user with respect to an
item---the state captures the instantaneous preference and the
response depends in a stochastic manner on the state.  The AOR {\em
  observes the response but not the state}. The second challenge is
that the act of {\em recommending content to the user itself changes
  the user's state of mind} (e.g., fatigued, stimulated) which in turn
influences the user's responses to future content. This means that the
model should allow for action dependent transitions between the
states. A third challenge is adapting to the heterogeneity among
users---two users may have not only different propensities towards a
content item but also different rates at which they react dynamically
to recommendations.  This in turn necessitates {\em learning the
  associated state transition models} under uncertainty of not knowing
  what state induced a response.

  This paper frames the problem of optimal recommendation under user
  adaptation and uncertainty as learning a stylized average-cost
  partially observable Markov decision process (POMDP). For a given
  user, an independent POMDP is associated with each item. The state
  of the POMDP expresses either a high (state 1) interest or a low
  (state 0) interest of the user for the item.  In each step, the
  AORS recommends one item and the user response for the item is
  determined by the state. This binary response is also available to
  the AORS. The states of each of the POMDPs changes accordingly as
  the the item is recommended or not recommended.  Thus the AORS can
  be seen to be a restless hidden Markov multi-armed bandit. In this
  paper we develop this model and describe a Thompson sampling
  mechanism to learn the parameters of the model using the response
  for each recommendation.  Specifically, our contributions in this
  paper are as follows.
\begin{enumerate}
\item Formulate a POMDP-based model for each item in the database 
of a recommendation system by incorporating user adaptation, hidden state 
and uncertainty in model. The AOR itself is modeled as an average cost 
restless hidden Markov multi-armed bandit.
\item Analyze the structure of the single-armed POMDP and show that
  the optimal policy is of single-threshold type in the belief.  A
  consequence of the single threshold is that the optimal policy has a
  cyclic form with a recommendation step followed by $k$
  no-recommendation steps. The optimal $k$ is derived.
\item Devise a natural online learning algorithm based on Thompson
  sampling (TS) for optimizing reward in the POMDP with no knowledge
  of the model parameters.
\item Derive what is, to our knowledge, the first known regret bounds
  for TS for online learning in a class of POMDPs.
\end{enumerate}

\subsection{Related Work}

Multi-armed bandit models for recommendation systems and for online
advertising have been modeled as contextual bandits, e.g.,
\cite{Langford07,Caron12,Li10} and the user interests are assumed to
be independent of the recommendation history, i.e., they have static
models of reward. There are several models for restless multi-armed
bandits that use state transitions and state-based rewards with
applications in dynamic spectrum access, e.g.,
\cite{LiuZhao10,Ouyang11, Li13}. Such models assume (1)~perfect
observation of the state when the arm is sampled, and (2)~state
transitions being independent of/unrelated to actions. There is some
work in modeling changing rewards in multi-armed bandits, e.g.,
\cite{Mansourifard12} but it is again in the fully observable state
case, thus circumventing the critical problem of state uncertainty
arising in user adaptation.  Other approaches towards handling user
reactions to recommendations have considered algorithms that use a
finite sequence of past user responses as a basis for deciding the
current recommendation, e.g., \cite{Hariri12}; but these are primarily
numerical studies. A more general framework for a restless multi-armed
bandit with unobservable states and action-dependent transitions was
considered in \cite{Meshram15,Meshram16a}. In \cite{Meshram16a} it was
shown that the such a system is approximately Whittle-indexable. The
restless bandit that we propose in this paper is a special case of
that from \cite{Meshram16a} for which we obtain much stronger results
and also a Thompson sampling-based algorithm to learn the parameters
of the arms.

The rest of the paper is organized as follows. In the next section we
describe the model and set up notation. In
Section~\ref{sec:optimal-policy} we analyze the structural properties
of the average-cost POMDP corresponding to a single arm. In
Section~\ref{sec:thompson} a Thompson sampling based algorithm is
described to learn the parameters of the POMDP based on the observed
reward. In Section~\ref{sec:regret-bound} we analyse the regret as
compared to the optimal policy. We conclude with some illustrative
numerical results and a discussion on extension to the multi-armed
bandit case.

\section{Model Description and Preliminaries}
\label{sec:model}

The AOR system is modeled as a restless multi-armed bandit with arm $i$
representing the state of the user with respect to item $i.$ Each arm
is modeled as an independent partially observable Markov decision
processes, with two states and two actions. We first describe the
model for a single generic arm. 

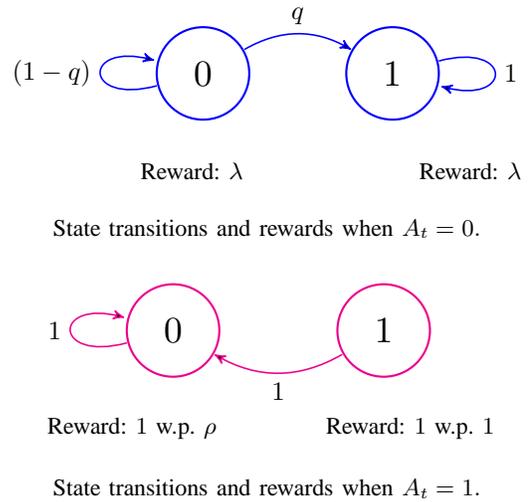
\begin{figure}
  \begin{center}
    \begin{tikzpicture}[draw=blue,>=stealth', auto, semithick, node distance=1.8cm]
      \tikzstyle{every state}=[fill=white,draw=blue,thick,text=black,scale=1.4]
      \node[state]    (A)                     {$0$};
      \node[state]    (B)[ right of=A]   {$1$};
      \path
      (A)  edge[loop left]     node{$(1-q)$}
      (A)
      edge[bend left,above,->]      node{$q $ }      (B)
      (B)  edge[loop right]    node{$1$}
      (B);
      \draw (-0.2,-1.3) node {\small{ Reward: $\lambda$ }};
      \draw (3.5,-1.3) node {\small{ Reward: $\lambda$ }}; 
    \end{tikzpicture}
  \end{center}

  \begin{center}
   { \small{State transitions and rewards when $A_t=0.$}}
  \end{center}
  
  \vspace{5pt}
  
  \begin{center}
    \begin{tikzpicture}[draw=magenta, >=stealth', auto, semithick, node distance=2cm]
      \tikzstyle{every state}=[fill=white,draw=magenta,thick,text=black,scale=1.4]
      \node[state]    (A)                     {$0$};
      \node[state]    (B)[ right of=A]   {$1$};
      \path
      (A) edge[loop left]     node{$1$}         (A)
      (B)     edge[bend left,below,->]      node{$1 $}         (A);    
      \draw (-0.6,-1.3) node {\small{ Reward: $1$ w.p. $\rho$} };
      \draw (3.1,-1.3) node {\small{ Reward: $1$  w.p. $1$} };
    \end{tikzpicture}
  \end{center}

  \begin{center}
    {\small{State transitions and rewards when $A_t=1.$}}
  \end{center}
  \caption{Single-arm POMDP. The state transitions and rewards for
    when the item is recommended and when it is not recommended.}
  \label{Fig:state-transitions}  
\end{figure}

$\mathcal{S} = \{0,1\}$ is the set of states with $0$ corresponding to
a low interest and $1$ corresponding to high interest.
$\mathcal{A}= \{ 0, 1\}$ is the set of actions with $1$ for sampling
the arm and $0$ for not sampling the arm.
$r:\mathcal{S} \times \mathcal{A} \rightarrow \mathcal{R}$ is the
reward function with $\mathcal{R} =\{0,1, \lambda \}$ being the set of
rewards. Time progresses in discrete steps indexed by
$t = 1, 2, 3, \ldots$ A step is an instant at which a recommendation
can potentially be made by the AORS. $X_t$ denotes the state of the
arm at the start of time step $t,$ $A_t$ is the action played at time
$t$ and and $R_t$ is the reward obtained at time $t.$

The reward structure for the POMDP is as follows. A unit reward is
obtained with probability $1$ if $A_t=X_t=1.$ This corresponds to
recommending the item when the user's interest is high. A unit reward
is obtained with probability $\rho$ if $A_t=1$ and $X_t=0,$ i.e., if
an item is recommended and the interest is low.  A reward $\lambda$ is
obtained independent of $X_t$ for $A_t = 0,$ i.e., when the item is
not recommended.
\begin{remark}
  A restless multi-armed bandit is typically analyzed by first
  analyzing the single arm case with a reward of $\lambda$ for
  $A_t=0.$ $\lambda$ is called the subsidy for not sampling. Such an
  analysis is used to determine if the Whittle index based policy can
  be used to optimally choose the arm at each time step. This analysis
  is also used to calculate the Whittle-index for an arm. 
\end{remark}
If $A_t=0,$ then $0 \rightarrow 1$ transitions occur with probability
$q$ and $1 \rightarrow 0$ transitions with probability $0.$ This
corresponds to an increasing interest as the time progresses since the
last recommendation. If $A_t=1,$ then $0 \rightarrow 0$ and
$1 \rightarrow 0$ transitions happen with probability $1$
corresponding to a decreased interest after a recommendation.
Fig.~\ref{Fig:state-transitions} illustrates the preceding discussion
in the form of a state transition diagram.

Recall that the actual state (0 or 1) is never observed. Let $\pi_t=
\prob{X_t=0},$ denote the belief about the state of the arm at the
beginning time $t.$ 
Let $H_t$ denote the history of actions and rewards up to time $t.$
Let $\phi_t: H_t \to \{0,1\}$ be the strategy that determines the
action at time $t.$ For a strategy $\phi$ and an initial belief $\pi$
at time $t=0$ (i.e., $\prob{X_0 = 0} = \pi$), the expected finite
horizon reward function is 
\begin{eqnarray*}
  V_T^{\phi}(\pi) & = & \mathbb{E}^{\phi}  \left[\sum_{t=0}^{T-1}  r(X_t,A_t)  
    \bigg \vert \pi_0 = \pi \right] \\
  & = & \mathbb{E}^{\phi}\left[ \sum_{t=0}^{T-1} \left[\pi_t \rho + 
      (1- \pi_t) 1 \right] \bigg \vert \pi_0 = \pi \right]. 
\end{eqnarray*}
The long term average reward is defined as $V^{\phi}(\pi) = \lim_{T
  \rightarrow \infty} \frac{1}{T}V_T^{\phi}(\pi).$ 

In the next section we assume that $(q, \rho),$ are known and
determine the policy $\phi$ that maximizes $V^{\phi}(\pi).$

\section{Single Arm: Optimal Policy }
\label{sec:optimal-policy}

We solve the average reward problem described in the previous section
by the vanishing discount approach \cite{Ross93,Avrachenkov15}---by
first considering a discounted reward system and then taking limits as
the discount vanishes. The infinite horizon discounted reward under
policy $\phi$ and discount $\beta,$ $0 < \beta < 1,$ is
\begin{eqnarray}
  V_{\beta}^{\phi}(\pi) & := & \mathbb{E}^{\phi} \left[ 
    \sum_{t=1}^{\infty} \beta^{t-1} \left( a^{\phi}_t(\pi_t \rho + (1-\pi_t)) +
    \right. \right. \nonumber \\
  && \left. \left. (1-a^{\phi}_t) \lambda \right) \bigg \vert \pi_0=\pi \right]. 
  \label{eq:SAB-valfn_disc}
\end{eqnarray}
From \cite{Meshram16a}, we can show that the following dynamic program
solves \eqref{eq:SAB-valfn_disc}. 

{\small
\begin{eqnarray}
  V_{\beta}(\pi)   =  \max \left\{ \lambda + \beta V_{\beta}((1-q)\pi),
    1- \pi(1-\rho) + \beta V_{\beta}(1) \right\} \hspace{-5pt}
  \label{eqn:V-of-pi}
\end{eqnarray}
}

Further, we can state the following about $V_{\beta}(\pi).$
\begin{lemma}
  \begin{enumerate}
  \item Equation~\eqref{eqn:V-of-pi} has a unique solution
    $V_{\beta}(\pi).$ Further, $V_{\beta}(\pi)$ is continuous and
    bounded.
  \item $V_{\beta}(\pi)$ is convex non-increasing in $\pi$ and
    increasing in $\beta.$
  \item $\big\vert V_{\beta}(\pi_1)- V_{\beta}(\pi_2)\big\vert <
    (1-\rho)$ for all $\pi \in [0,1]$ and $\beta \in [0,1).$
  \item The optimal policy is of threshold type with a single
    threshold for $\beta \in [0, 1)$ and $\lambda_L \leq \lambda \leq
    \lambda_H.$
  \end{enumerate} 
  \label{lemma:props-of-V-of-pi}
\end{lemma}
The first two above follow directly from \cite{Meshram16a} and the last
two are derived in \cite{Meshram16b}. 

Define $\overline{V}_{\beta} := V_{\beta}(\pi) - V_{\beta}(1)$ for
$\pi \in [0,1].$  From \eqref{eqn:V-of-pi}, we get
\begin{eqnarray}
  \overline{V}_{\beta} + (1- \beta)V_{\beta}(1)  & = & \max \left\{ 
    \lambda + \beta \overline{V}_{\beta}((1-q)\pi),  \right. \nonumber \\
  && \left. 1- \pi(1-\rho) \right\}
\label{eqn:V-bar-of-beta}
\end{eqnarray}

From Lemma~\ref{lemma:props-of-V-of-pi}, $\overline{V}_{\beta}(\pi)$
is convex monotone in $\pi$ and by definition
$\overline{V}_{\beta}(1) = 0.$ Further, from the lemma we know that
there is a constant $C < \infty$ such that
$\big\vert V_{\beta}(\pi) - V_{\beta}(1) \big\vert < C.$ This implies
that $\overline{V}_{\beta}(\pi)$ is bounded and Lipschitz-continuous.
Finally, $(1-\beta)V_{\beta}(\pi)$ is also bounded.  Hence we can
apply the Arzela-Ascoli theorem \cite{rudin-principles}, to find a
subsequence $(\overline{V}_{\beta_k}(\pi), (1-\beta)V_{\beta_k}(\pi))$
that converges uniformly to $(V(\pi),g)$ as $\beta_k \rightarrow 1.$
Thus, as $\beta_k \rightarrow 1,$ along an appropriate subsequence,
\eqref{eqn:V-bar-of-beta} reduces to
\begin{eqnarray}
  V(\pi) + g = \max \left\{\lambda+ V((1-q)\pi), 1 -\pi(1-\rho) \right\}, 
  \label{eq:dynamic-prog-avgc}
\end{eqnarray}
for all $\pi \in [0,1].$ \eqref{eq:dynamic-prog-avgc} is the dynamic
programming equation whose solution gives us the optimal value
function that maximizes the average reward.

Since $V(\pi)$ inherits the structural properties of $V_{\beta}(\pi),$
we have, from Lemma~\ref{lemma:props-of-V-of-pi}, that
\begin{lemma}
\label{lem:inherit}
\begin{enumerate}
\item $V(\pi)$ is monotone non-increasing and convex in $\pi.$ 
\item The optimal policy is of threshold type with a single threshold 
  for $\lambda_L \leq \lambda  \leq \lambda_H.$
\end{enumerate}
\end{lemma}

This in turn leads us to the following theorem which is a direct analog of
Theorem~6.17 in \cite{Ross93}.
\begin{theorem}
  If there exists a bounded function $V(\pi)$ for $\pi \in [0,1]$ and
  a constant $g$ that satisfies \eqref{eq:dynamic-prog-avgc}, then
  there exists a stationary policy $\phi^*$ such that
  \begin{eqnarray}
    g = \max_{\phi}\lim_{T \rightarrow \infty} \frac{1}{T} V_{T}^{\phi}(\pi)
  \end{eqnarray}
  for all $\pi \in [0,1]$, and moreover, $\phi^*$ is the policy for
  which the RHS of \eqref{eq:dynamic-prog-avgc} is maximized.
  \label{thm:avg-reward-opt-policy}
\end{theorem}

\subsection{An Equivalent Form for the Optimal Policy}
\label{sec:equivopt}
For the single-armed bandit, the threshold policy of
Lemma~\ref{lem:inherit} can be interpreted as follows. Let $\pi_T$ be
the threshold such that the optimal policy is $A_t=1$ if
$\pi_t \leq \pi_T$ and $A_t=0,$ if $\pi_t > \pi_T.$ We know that if
$A_t=1,$ then $\pi_{t+1}=1,$ i.e., if the item is recommended then the
state becomes 0. When the item is not recommended, the belief about
state $0$ decreases by a factor of $(1-q).$ Since there is a single
threshold and $\pi_t$ monotonically decreases every time the item is
not recommended, the optimal policy will be to wait for $k$ steps
before recommending again, where $k$ is the first time that $\pi_t$
has crossed $\pi_T.$ This value $k$ is a function of $q$ and $\rho$
and will be denoted by $k_{pt}(q, \rho).$

We first consider infinite horizon discounted reward problem. In this
case, solving \eqref{eq:SAB-valfn_disc} is equivalent to solving
following optimization problem.
\begin{eqnarray}
  k_{\beta,opt}(q,\rho)  = \arg\max_{k \geq 1} \tilde{V}_{\beta}(k),
\end{eqnarray}
where $\tilde{V}_{\beta}(k)$ is the value function obtained by not
recommending for $k$ steps between successive recommendations. We can
write
\begin{eqnarray*}
  \tilde{V}_{\beta}(k) & := &\left\{ \lambda+ \lambda \beta + \lambda \beta^2 + 
    \cdots + \lambda \beta^{k-1} \right. \\
  && \hspace{10pt} \left. + \beta^{k} \left[ (1-q)^{k} \rho + 1 - 
      (1-q)^k \right] \right\} \\
  &&+ \beta^{k+1} \left\{ \lambda+ \lambda \beta + \lambda \beta^2 + 
    \cdots + \lambda \beta^{k-1} \right. \\
  && \hspace{10pt} \left. + \beta^{k} \left[ (1-q)^{k} \rho + 1 - 
      (1-q)^k \right] \right\} \\
  && + \beta^{2(k+1)} \left\{ \lambda+ \lambda \beta + \lambda \beta^2 + 
    \cdots + \lambda \beta^{k-1} \right. \\
  && \hspace{10pt} \left. + \beta^{k} \left[ (1-q)^{k} \rho + 1 - (1-q)^k \right] 
  \right\}  + \cdots
\end{eqnarray*}
Let $C_k$ denote the reward from the first one cycle of $k$ steps with
no recommendations for $(k-1)$ steps and a recommendation in the
$k$-th step. We can write 
\begin{eqnarray*}
  && C_k := \  \lambda+ \lambda \beta + \lambda \beta^2 + \cdots + 
  \lambda \beta^{k-1} \\
  && \hspace{45pt} + \beta^{k} \left[ (1-q)^{k} \rho + 1 - (1-q)^k \right]  \\
  && \hspace{15pt} = \lambda \frac{(1-\beta^k)}{(1-\beta)} + \beta^k 
  \left[ (1-q)^{k} \rho + 1 - (1-q)^k \right]
\end{eqnarray*}
The first $k-1$ terms above correspond to the reward from not sampling
and the $k$th term denotes the reward from sampling. Thus,
$\tilde{V}_{\beta}(k)$ can be rewritten as follows.

{\small
\begin{eqnarray*}
 &&  \hspace{-10pt} \tilde{V}_{\beta}(k) \ = \ C_k \left[ 1+ \beta^{k+1} + \beta^{2(k+1)} + 
    \cdots \right] \\
  && \hspace{-10pt}  = \frac{1}{1- \beta^{k+1}} \left[ \lambda \frac{(1-\beta^k)}{(1-\beta)} 
    + \beta^k \left[ (1-q)^{k} \rho + 1 - (1-q)^k \right] \right] \\
\end{eqnarray*}
}

\vspace{-1.5\baselineskip}
The preceding discussion gives us the following result on the value function
and the optimal policy for the average reward criterion POMDP. 

\begin{theorem}
\begin{enumerate}
\item The value function for policy $k$ is
\begin{eqnarray*}
\hspace{-10pt}  \tilde{V}(k)  &=& \lim_{\beta \rightarrow 1} (1-\beta)\tilde{V}_{\beta}(k) \\
  &=& \frac{1}{{k+1}} \left[ \lambda k + 
    \left[ (1-q)^{k} \rho + 1 - (1-q)^k \right] \right].
\end{eqnarray*}
\item The optimum policy $k_{opt}(q,\rho)$ satisfies 

{\small
\begin{eqnarray}
 && \hspace{-35pt}  k_{opt}(q, \rho) \ = \ \arg\max_{k \geq 1} \tilde{V}(k) \nonumber \\
  &&\hspace{-35pt}  =\  \arg\max_{k \geq 1} \frac{1}{{k+1}} \left[ \lambda k + 
    \left[ (1-q)^{k} \rho + 1 - (1-q)^k \right] \right]
  \label{eqn:optimal-k}
\end{eqnarray}
}

\end{enumerate}
\end{theorem}

Thus, for the single armed bandit, given $q,$ $\rho$ and $\lambda,$ we
obtain the optimal policy as the number of steps to wait before
recommending the item again. In the next section we describe the
Thompson sampling algorithm to learn the parameters based on the
reward that is observed. Subsequently, we analyze the regret from the
learning process. We remind the reader that the state is never
observed in the system and the learning is based only on rewards.

\section{Thompson Sampling learning algorithm}
\label{sec:thompson}
We have seen that the optimal policy for the (single-arm) POMDP
described by $q$ and $\rho$ is of threshold type
(Section~\ref{sec:equivopt}), and corresponds to waiting for
$k_{\text{opt}}(q, \rho)$ steps in between successive
recommendations. However, when the parameters $q$, $\rho$ that
describe the Markov chain transition probabilities are {\em unknown} a
priori\footnote{as in an AOR system where user behavior is unknown at
  start}, they must be {\em learnt} or {\em inferred} from the
available feedback in order to attain maximum cumulative reward. This
section describes an online algorithm that learns to play the optimal
policy using experience, i.e., observations from previously played
actions, while at the same time keeping the net reward as high as
possible (the explore-exploit problem).

The learning algorithm (Algorithm~\ref{algo:TS}) is a version of the
popular Thompson sampling strategy \cite{Thompson}, developed
originally for stochastic multi-armed bandit problems
\cite{AgrawalG}, and subsequently extended to learning in
Markov Decision Processes (MDPs) \cite{OsbVanRoy14:eluderRL,Gopalan15,AbbPalCsa11:linbandits}
and POMDPs. It works in epochs, where an epoch is defined to be the
interval of time from an instant at which the POMDP is sampled (action
$1$ is played) up until the next instant at which it is sampled
again. At the beginning, the algorithm initializes a prior or belief
distribution\footnote{Note that the prior used in Thompson sampling is
  merely a parameter of the algorithm (e.g., the uniform measure over a
  compact domain), carefully designed to induce random exploration, and
  is not related to any Bayesian modeling assumptions on the true
  model as such.}  on the space of all candidate parameters/models,
which in our case is any subset $\mathcal{X}$ of the unit square
$[0,1] \times [0,1]$ containing all possible POMDPs parameterized by
$(q,\rho)$. At the start of each epoch $\ell \geq 1$, a model
$(q_\ell, \rho_\ell)$ is randomly and independently sampled according
to the current prior over $\mathcal{X}$ (this random draw is crucial
in inducing exploration over the model space). Then, the optimal
policy for this sampled model is computed, which by the previous
results corresponds to sampling the chain after an interval of
$k_{\text{opt}}(q_\ell, \rho_\ell)$ time instants\footnote{This could
  be carried out using either standard planning methods such as
  value/policy iteration or exhaustive search over threshold-type
  policies.}. This policy is now applied for one cycle, i.e., the
algorithm waits for the specified interval of time instants, samples
the chain at the next time instant, and obtains an observation for the
sampled instant (a Bernoulli-distributed reward). The observed reward
is used to update the prior over models via Bayes' rule, the epoch
ends, and the next epoch starts with the updated prior.

\noindent {\em Notation.} In Algorithm~\ref{algo:TS},
$\mathcal{B}(\mathcal{X})$ denotes the Borel $\sigma$-algebra of
$\mathcal{X} \subset \mathbb{R}^d$.  $\prob{R = r \given (q, \rho), k}$
denotes the likelihood, under the POMDP model specified by parameters
$(q,\rho)$, of observing a reward of $r \in \{0,1\}$ upon sampling the
Markov chain after having not sampled the chain for exactly $k$
previous time instants. %
%
%
\begin{algorithm}[htbp]
  \caption{Thompson sampling algorithm for learning the optimal
    policy}
  \begin{algorithmic}
     \STATE {\bf Input:} Parameter space $\mathcal{X} \subseteq
     [0,1]^2$, Policy space $\mathcal{K} \subset \{0, 1, 2, \ldots\}$,
     Observation space $\mathcal{R} = \{0,1\}$, Prior probability
     distribution $Z_1$ over $(\mathcal{X}, \mathcal{B}(\mathcal{X}))$
  
  
     \FOR {epoch $ \ell = 1, 2, \ldots$ }
     
     \STATE {\bf Sample} $(q_{\ell},\rho_{\ell}) \in \mathcal{X}$ according to the probability distribution $Z_{\ell}$
     \STATE {\bf Compute} the optimal policy $k_{\ell} = k_{\text{opt}}(q_{\ell}, \rho_{\ell})$ for sampled parameters
     \STATE {\bf Apply} the policy $k_{\ell}$ once: wait for the next $(k_{\ell}-1)$ time steps and sample the Markov chain at the $k_{\ell}$-th time instant
     \STATE {\bf Observe} reward on sampling, denote it by $R_{\ell} \in \{0,1\}$
     \STATE {\bf Update} the current prior over $(q, \rho)$ to
     \begin{eqnarray*}
       Z_{\ell+1}(B) := \frac{\int_B \prob{R = R_\ell \given (q, \rho), k_\ell} Z_\ell(q,\rho) dq \; d\rho}{\int_{\mathcal{X}} \prob{R = R_\ell \given (q, \rho), k_\ell} Z_\ell(q,\rho) dq \; d\rho}
     \end{eqnarray*}
     for any Borel set $B \in \mathcal{B}(\mathcal{X})$. 
  
     \ENDFOR
     
  \end{algorithmic}
  \label{algo:TS}
\end{algorithm}
%
Specifically, we have
\begin{eqnarray*}
\prob{R= r \given (q, \rho), k} = 
\begin{cases}
f(q,\rho,k)&  \mbox{if $r=1$} \\
1-f(q,\rho,k)& \mbox{if $r=0$},
\end{cases}
\end{eqnarray*}
where $f(q,\rho,k)$ is simply the probability of observing a
reward of $1$ after having waited for $k$ time steps since the last
sample, when the parameters are $(q,\rho)$. It follows that
\begin{eqnarray*}
f(q,\rho,k) = (1-q)^k \rho  + \left[1 - (1-q)^k \right].
\end{eqnarray*}

\section{Main Result -- Regret Bound}
\label{sec:regret-bound}
In this section, we show an analytical performance guarantee for
Algorithm~\ref{algo:TS}.


To this end, we consider a widely employed measure of performance from
online learning theory, namely {\em regret}
\cite{BookCBL,JakschOA10}. The regret of a strategy, for
the POMDP described by $(q^*, \rho^*)$, is the difference between the
cumulative reward which the optimal policy\footnote{We overload
  notation, when the context is clear, to represent the optimal policy
  using the optimal waiting time $k_{\text{opt}}(q^*,\rho^*)$.}
$k_{\text{opt}}(q^*,\rho^*)$ earns when run from a fixed initial state
for a fixed time horizon $T$, and that which the strategy earns with
the same initial state and time horizon. Formally, the regret for a
strategy $\mathcal{A}$ is the random variable
\[ R^{\mathcal{A}}_{(q^*,\rho^*)}(T) \bydef \sum_{t=0}^{T-1}
r\left(X_t,A_t^{k_{\text{opt}}(q^*,\rho^*)}\right) - \sum_{t=0}^{T-1}
r\left(X_t,A_t^{\mathcal{A}}\right), \]
where $A_t^{k_{\text{opt}}(q^*,\rho^*)}$ (resp. $A_t^{\mathcal{A}}$)
represents the action taken by $k_{\text{opt}}(q^*,\rho^*)$
(resp. $\mathcal{A}$) at time instant $t$, and it is assumed that the
algorithms $\mathcal{A}$ and $k_{\text{opt}}(q^*,\rho^*)$ are run on
independent POMDP instances. The goal is typically to bound the regret
of a sequential decision making algorithm as a function of the
structure of the POMDP (number of states/actions in the underlying
MDP) and show that it grows only sub-linearly with time $T$ (i.e., the
per-round regret vanishes), either in expectation or with high
probability. 

Towards bounding the regret of the Thompson sampling POMDP
algorithm (Algorithm~\ref{algo:TS}), it is convenient to consider a
modified version of regret that essentially counts the number of
epochs during the operation of the algorithm in which the policy used
is not $k_{\text{opt}}(q^*,\rho^*)$, or in other words the length of
the epoch consisting of no-sampling instants is not
$k_{\text{opt}}(q^*,\rho^*)$. This corresponds to the quantity
\[ \tilde{R}_{(q^*,\rho^*)}(L) \bydef \sum_{\ell = 1}^{L} 1_{\{k_\ell
  \neq k_{\text{opt}}(q^*,\rho^*)\}}, \] defined for the first $L$
epochs that the algorithm executes. Note that under the reasonable
assumption that an upper bound $k_{\max}$ on
$k_{\text{opt}}(q^*,\rho^*)$ is available a priori, and if the
Thompson sampling algorithm samples at all times parameters $(q,\rho)$
for which $k_{\text{opt}}(q,\rho) \leq k_{\max}$, then the length of
each epoch is bounded between $1$ and $k_{\max}$; thus the standard
regret $R(T)$ is bounded in terms of the modified regret
$\tilde{R}(L)$ by a constant factor $k_{\max}$ (note that the maximum
possible reward is $1$). In order to focus on the order-wise scaling
of the regret with time or number of epochs, we henceforth concentrate
on bounding the (modified) regret $\tilde{R}(T)$, with high probability.

%

We will need the following set of mild assumptions on the structure of
the parameter space and the initial prior under which a regret bound
holds.

\begin{assumption}
  \label{ass:1}
  (a) The parameter space $\mathcal{X} \subseteq [\eta, 1-\eta]$ for
  some $\eta \in \left(0, \frac{1}{2}\right)$, (b)
  $|\mathcal{X}| < \infty$, (c) The true model
  $(q^*, \rho^*) \in \mathcal{X}$, (d) The prior distribution $Z_1$
  over $\mathcal{X}$ puts positive mass on the true model, (e) There
  is a unique (average-reward) optimal policy
  $k_{\text{opt}}(q^*, \rho^*) \leq k_{\max}$ for the true model with
  a known upper bound $k_{\max} \in \mathbb{Z}$.
\end{assumption}

\begin{theorem}[Main Result -- Thompson sampling regret]
  Let Assumption~\ref{ass:1} hold, and let
  $\epsilon, \delta \in (0,1).$ 
  There exists $L_0 \equiv L_0(\epsilon)$ such that the following
  bound holds, for the cumulative regret of Algorithm~\ref{algo:TS}
  with initial prior $Z_1$, with probability at least $1-\delta$ for
  all $L \geq L_0$:
\begin{eqnarray*}
\tilde{R}_{(q^*,\rho^*)}(L) \leq B + C (\log L),
\label{eqn:regret-bound}
\end{eqnarray*}
where $B \equiv B(\epsilon, \delta, (q^*,\rho^*), \mathcal{X})$ is a
problem-dependent constant independent of the number of epochs $L$,
and $C \equiv C(\delta, (q^*,\rho^*), \mathcal{X})$ is the solution to
an optimization problem \eqref{eq:objective}.
\label{thm:regret-bound-Th} 
\end{theorem}
{\em Note.} The optimization problem is described in detail in
Appendix \ref{proof:regret-bound-Th} for the sake of clarity.

\noindent {\em Discussion.} Theorem~\ref{thm:regret-bound-Th}
establishes that the regret of Algorithm~\ref{algo:TS} scales only
logarithmically (thus, sub-linearly) with time (or epochs), with high
probability, when starting with a `grain-of-truth' prior that ascribes
positive probability to the true model. The algorithm is thus able to
achieve a suitable balance between exploring across different sampling
policies and exploiting its improving knowledge about the true model
$(q^*, \rho^*)$ to keep the regret controlled, in a nontrivial
fashion. Moreover, this is achieved in a POMDP model in which the
state of the Markov chain is never available at any time instant, but
instead only a stochastic reward correlated with the current state is
observed that conveys implicit information about the true parameter
$(q^*, \rho^*)$. The logarithmic growth of the regret with time is
controlled by the quantity $C$, which depends on Kullback-Leibler (KL)
divergences of the distribution of the observations under different
models and policies, thus encoding the information structure of the
observations.

The theorem is proven by following closely the method developed to
show \cite[Theorem $1$ and Proposition $2$]{GopManMan14:thompson} and
\cite[Theorem $1$ and $5$]{Gopalan15}, namely the strategy of bounding
the posterior mass (with high probability) both from below (in a
neighborhood of the true model) and from above (outside the
neighborhood, for parameters corresponding to suboptimal policies)
spelt out in detail in \cite[Appendix A]{GopManMan14:thompson}. We
describe the derivation of Theorem~\ref{thm:regret-bound-Th} in
Appendix~\ref{proof:regret-bound-Th}. 

 %

The following accompanying result provides more insight into the
order-wise scaling of regret. In the following,
$D(p||q) \bydef p \log \left(\frac{p}{q} \right) + (1-p) \log
\left(\frac{1-p}{1-q} \right)$
denotes the KL divergence between Bernoulli distributions of parameter
$p$ and $q$, $0 < p, q < 1$.

\begin{theorem}
\label{thm:reduction}
Consider $L$ to be large enough so that
\begin{eqnarray*}
\max_{(q,\rho) \in \mathcal{X}, k \leq k_{\max}} D(f(q^*,\rho^*,k)||f(q,\rho,k)) \leq \frac{1+\epsilon}{1-\epsilon} \log L. 
\end{eqnarray*}
Then, there exists $\Delta_2 > 0$ such that
\[ C(\log L) \leq \left(\frac{1}{\Delta_2} \right)
\frac{2(1+\epsilon)}{1-\epsilon} \log L.\]
\end{theorem}

{\em Discussion.} Theorem~\ref{thm:reduction} highlights a key
property of the regret induced by the information structure of the
POMDP problem -- that the regret asymptotically does not scale with
the total number $k_{\max}$ of candidate optimal policies, which could
be large by itself. This can be contrasted with running a simple
multi-armed bandit algorithm such as UCB \cite{AuerCF02} with the
`arms' being different waiting-duration policies with the duration
ranging from $0, 1, \ldots, k_{\max}$ (a total of $ k_{\max} + 1$
arms), and the reward being the reward from applying a single cycle of
any such policy while disregarding the POMDP structure entirely. It
follows from standard stochastic bandit regret bounds that such an
algorithm would achieve regret that scales with the total number of
arms, i.e., $O(k_{\max} \log T)$. The advantage of using Thompson
sampling with a prior on POMDP structures $(q, \rho)$ is that every
application of any waiting-time policy provides a non-trivial amount
of information (in the sense of the prior-posterior update) about the
true POMDP $(q^*, \rho^*)$, and hence about all other policies (in the
multi-armed bandit view this is akin to any arm providing reward
information about all arms following every pull). The proof of the
result is motivated by \cite[Proposition 2]{GopManMan14:thompson}, and
is detailed in the appendix.

\section{Numerical Results and Discussion}
\label{sec:numerical-discuss}

\begin{figure*}
  \begin{center}
    \begin{tabular}{cccc}
      \includegraphics[scale=0.35]{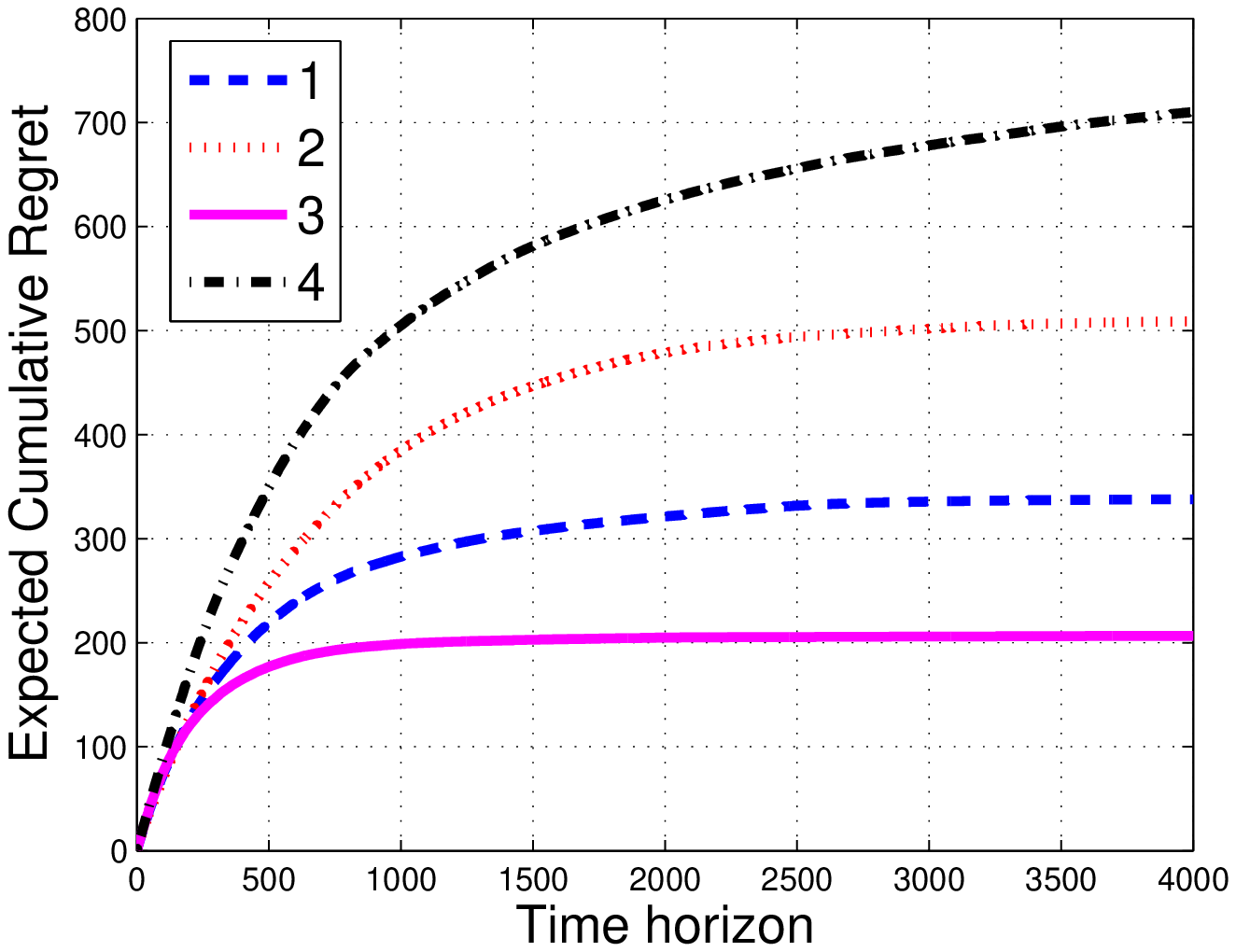}
      & 
      \includegraphics[scale=0.35]{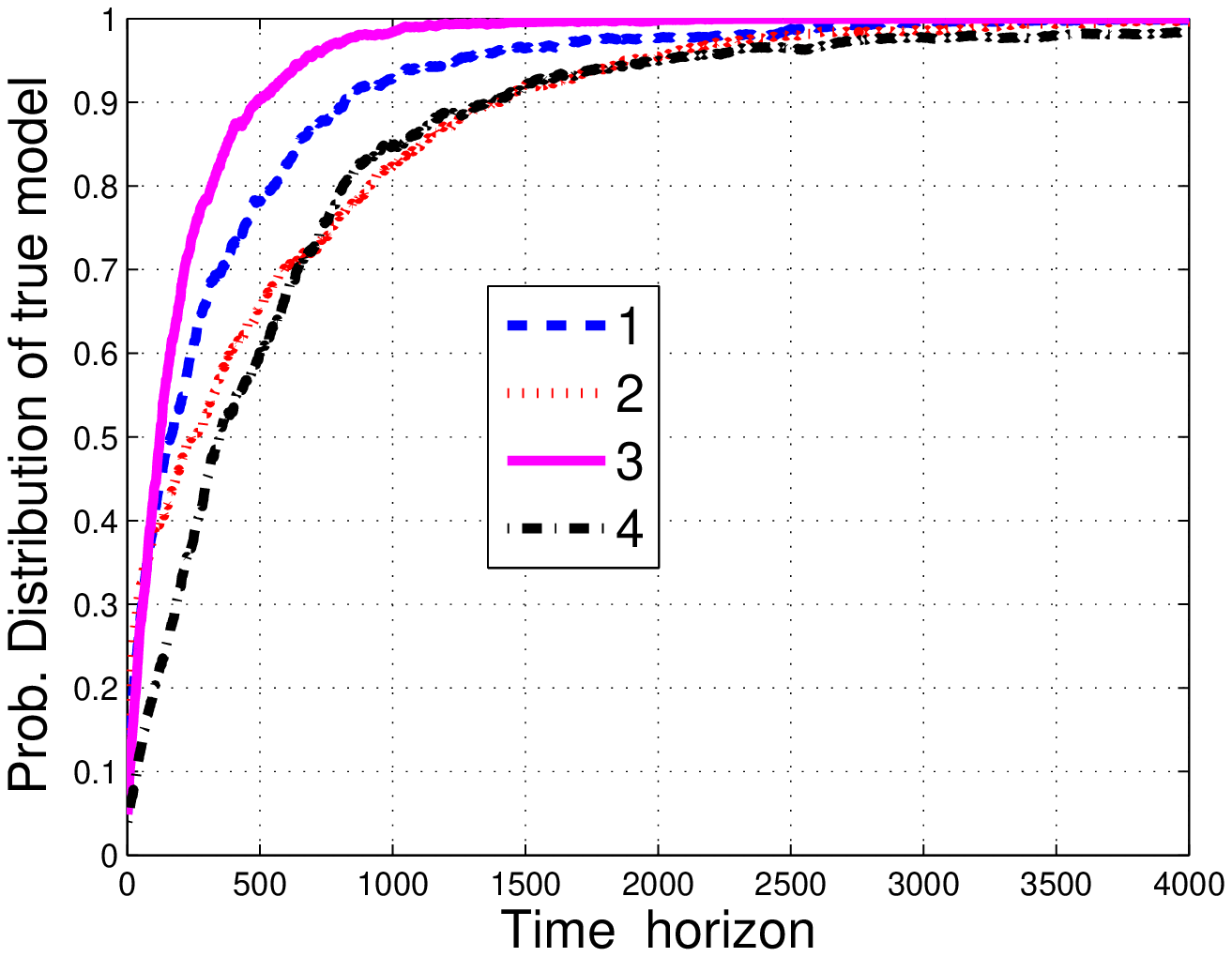}  
     \end{tabular}
  \end{center}
  \caption{Regret vs time horizon and probability mass on the true
    value vs time horizon for coarse grid. }
  \label{Fig-Regret1-Prob1}
\end{figure*}
\begin{figure*}
  \begin{center}
    \begin{tabular}{cccc}
      \includegraphics[scale=0.35]{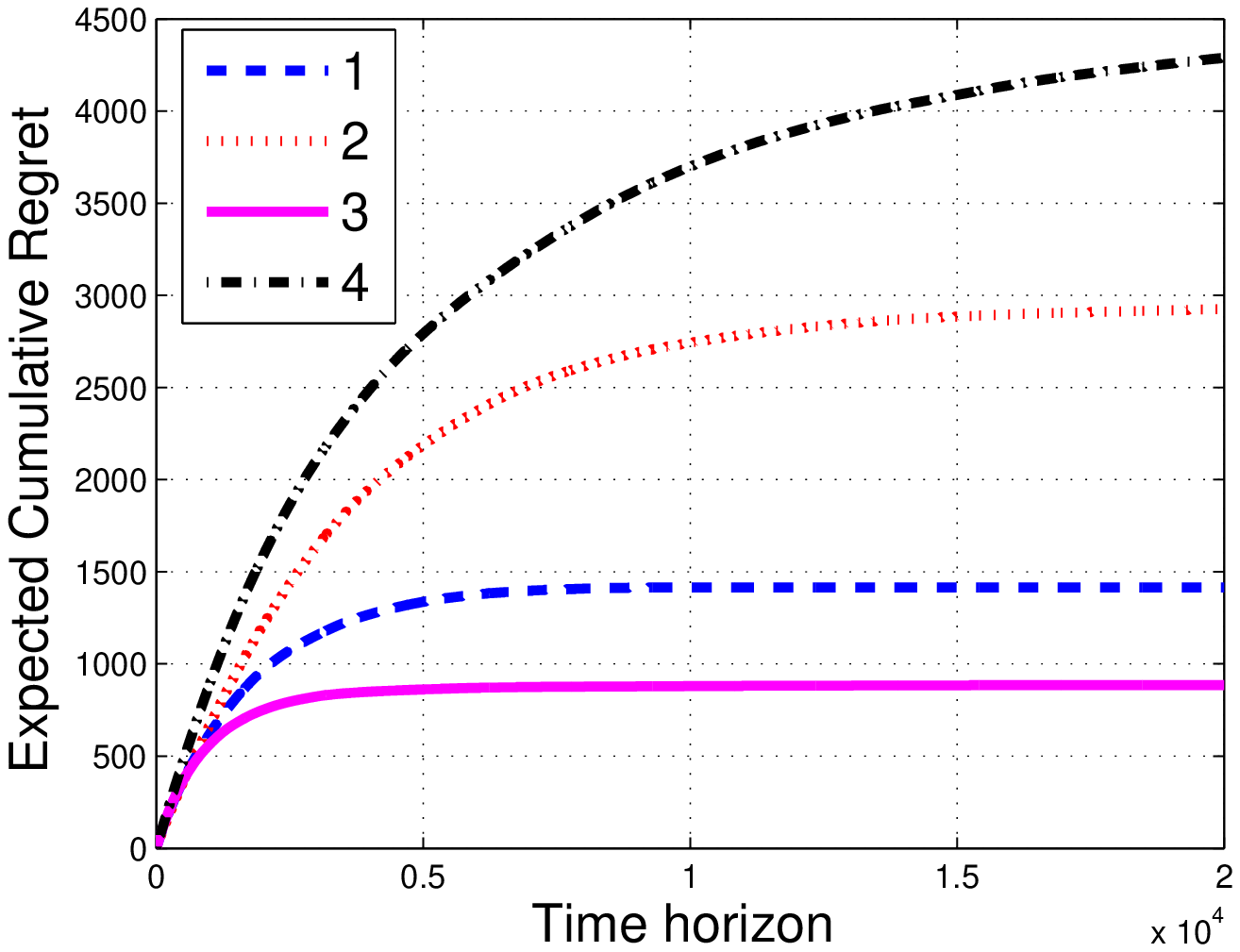}
      & 
      \includegraphics[scale=0.35]{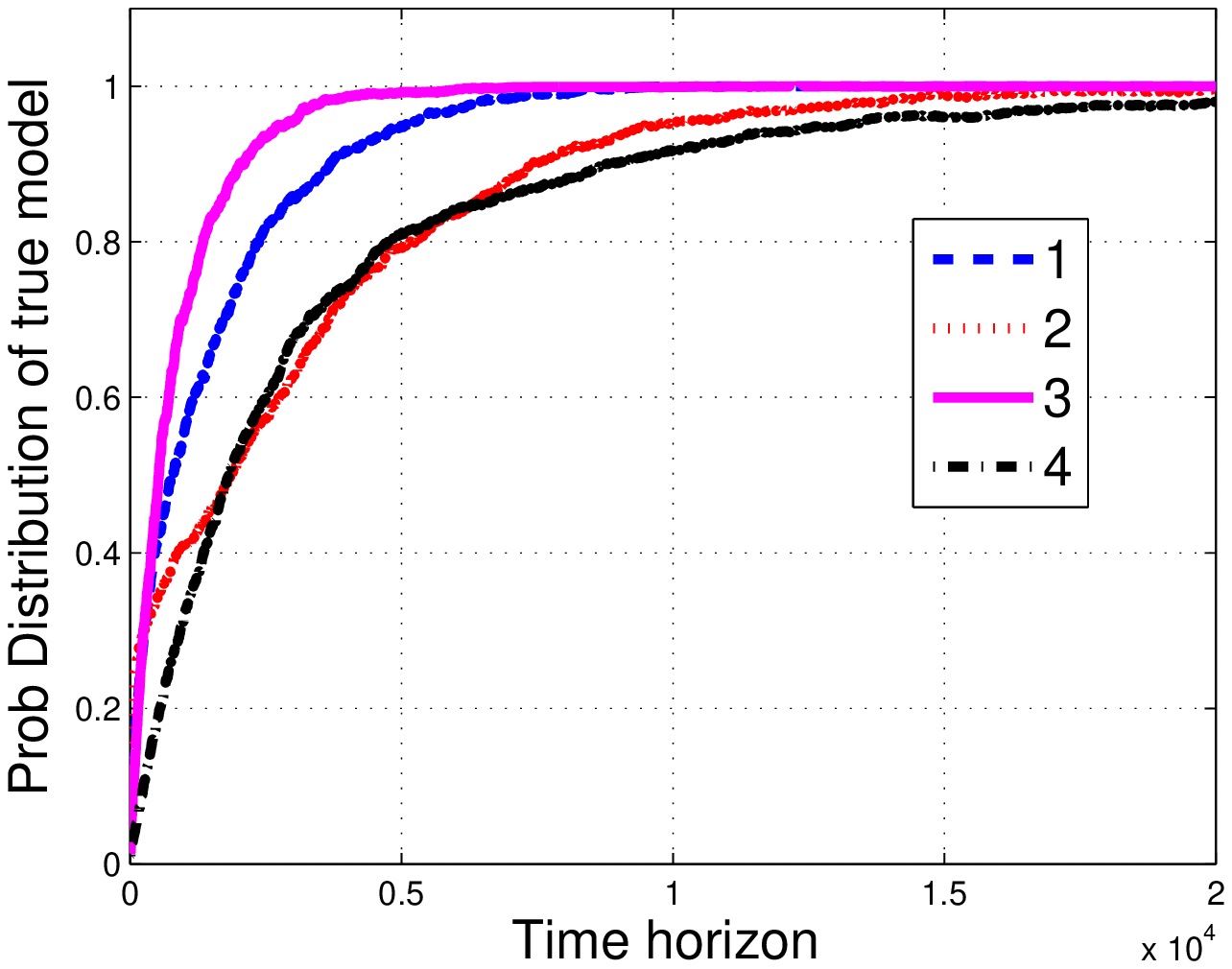}  
     \end{tabular}
  \end{center}
  \caption{Regret vs time horizon and probability mass on the true
    value vs time horizon for finer grid. }
  \label{Fig-Regret2-Prob2}
\end{figure*}
We present some numerical results to show the rate of convergence of
Algorithm~\ref{algo:TS} to the optimal policy of the POMDP model, for
various configurations of the true model and initial prior. We fix
$\lambda=0.3$ in all of the simulations and consider four combinations
of the true models $(q^*,\rho^{*});$ (1)~(0.05,0.25), \
(2)~(0.05,0.15), \ (3)~(0.05,0.35), and (4)~(0.15,0.35). For each of
these four values of the true parameters, we present two performance
measures as a function of the time step---the regret and also
probability mass on the true values. These values are average from 300
runs of the simulation.  The first set of plots (shown in
Fig.~\ref{Fig-Regret1-Prob1}) are obtained by discretising the
parameter space coarsely into a $(5 \times 5)$ grid at
$(0.05, 0.15, 0.25, 0.35, 0.45)$ and starting with the uniform a
distribution on the 25 points. The second set of plots (shown in
Fig.~\ref{Fig-Regret2-Prob2}) is obtained by using a finer
$10 \times 10$ grid at $(0.05, 0.10, \ldots, 0.50).$

We observe that true model and the initial prior (i.e., supported on
the coarse/fine grid) both affect the convergence rate of regret and
probability distribution of true model. The effect of the prior on the
finer grid is to increase the overall regret and slow down the
convergence to the true model. This is presumably due to (a) the fact
that imposing a prior over a fairly coarse grid is equivalent to
providing a large amount of side information about the true model
(i.e., that it must be one of a small set of models), and on a related
note, (b) the presence of confounding or competing models that are
closer to it than in the coarse grid prior, which must be eliminated
to achieve low regret (this is analogous to the phenomenon of smaller
`gap' in multi-armed bandits leading to higher regret).

\subsection{Discussion and Directions -- Multi-armed bandit case}
We formulated the problem of optimizing recommendations for a single
user, whose taste in a certain item changes with recommendations, as
online learning in a two-state, two-parameter POMDP. Using this
approach, we developed and analyzed the performance of a natural
Thompson-sampling algorithm for learning the optimal policy for a
user-item pair. A logical next step in this investigation is to treat
the multi-armed bandit version of the problem with {\em multiple}
independently evolving POMDPs, each representing different
users/items, and a resource constraint on which users/items can be
activated at any instant, e.g., decide which of several items is to be
shown to a user at a certain time, given that the user remembers how
far ago she consumed a certain item and may respond accordingly to
item recommendations.

The Thompson sampling based algorithm proposed in this work could be
extended to cover the multi-armed bandit case by jointly sampling
parameters for all POMDPs, computing the optimal refresh rate for each
of them, and scheduling the recommendations accordingly while at the
same time updating its current prior to incorporate observations. This
opens up new avenues for analysis of performance for such algorithms,
and we plan to pursue it as part of future work.

\bibliographystyle{IEEEbib}

\bibliography{restless-bandits}

\begin{thebibliography}{10}

\bibitem{Langford07}
J.~Langford and T.~Zhang,
\newblock ``The epoch-greedy algorithm for contextual multi-armed bandits,''
\newblock in {\em Advances in Neural Information Processing Systems}. NIPS,
  Dec. 2007, pp. 1--8.

\bibitem{Caron12}
S.~Caron, B.~Kveton, M.~Lelarge, and S.~Bhagat,
\newblock ``Leveraging side observations in stochastic bandits,''
\newblock {\em Arxiv}, 2012.

\bibitem{Li10}
L.~Li, W.~Chu, J.~Langford, and R.~E. Schapire,
\newblock ``A contextual-bandit approach to personalized news article
  recommendation,''
\newblock in {\em Conference on World Wide Web}. ACM, April 2010, pp. 661--670.

\bibitem{LiuZhao10}
K.~Liu and Q.~Zhao,
\newblock ``Indexability of restless bandit problems and optimality of
  {W}hittle index for dynamic multichannel access,''
\newblock {\em IEEE Transactions Information Theory}, vol. 56, no. 11, pp.
  5557--5567, November 2010.

\bibitem{Ouyang11}
W.~Ouyang, S.~Murugesan, A.~Eyrilmaz, and N.~Shroff,
\newblock ``Exploiting channel memory for joint estimation and scheduling in
  downlink networks,''
\newblock in {\em Proceedings of {IEEE INFOCOM}}, 2011.

\bibitem{Li13}
C.~Li and M.~J. Neely,
\newblock ``Network utility maximization over partially observable {M}arkovian
  channels,''
\newblock {\em Performance Evaluation}, vol. 70, no. 7--8, pp. 528--548, July
  2013.

\bibitem{Mansourifard12}
P.~Mansourifard, T.~Javidi, and B.~Krishnamachari,
\newblock ``Optimality of myopic policy for a class of monotone affine restless
  multi-armed bandits,''
\newblock in {\em Proceedings of {IEEE CDC}}, 2012.

\bibitem{Hariri12}
N.~Hariri, B.~Mobasher, and R.~Burke,
\newblock ``Context-aware music recommendation based on latent topic sequential
  patterns,''
\newblock in {\em Proceedings of the {S}ixth ACM {C}onference on {R}ecommender
  {}ystems {(RecSys '12)}}, 2012, pp. 131--138.

\bibitem{Meshram15}
R.~Meshram, D.~Manjunath, and A.~Gopalan,
\newblock ``A restless bandit with no observable states for recommendation
  systems and communication link scheduling,''
\newblock in {\em Proceedings of {IEEE CDC}}, 2015.

\bibitem{Meshram16a}
R.~Meshram, D.~Manjunath, and A.~Gopalan,
\newblock ``On the {W}hittle index for restless multi-armed hidden markov
  bandits,'' Submitted for Publication. Also available on {A}rxiv:1603.04739,
  2016.

\bibitem{Ross93}
S.~M. Ross,
\newblock {\em Applied Probability Models with Optimization Applications},
\newblock Dover Publications, 1993.

\bibitem{Avrachenkov15}
K.~Avrachenkov and V.~S. Borkar,
\newblock ``Whittle index policy for crawling ephemeral content,''
\newblock Tech. {R}ep. Research Report No.~8702, INRIA, 2015.

\bibitem{Meshram16b}
R.~Meshram, D.~Manjunath, and A.~Gopalan,
\newblock ``Adaptive playlists from hidden markov bandits,'' Submitted for
  Publication. Also available arxiv, 2016.

\bibitem{rudin-principles}
Walter Rudin,
\newblock {\em Principles of mathematical analysis},
\newblock McGraw-Hill Book Co., New York, third edition, 1976,
\newblock International Series in Pure and Applied Mathematics.

\bibitem{Thompson}
William~R. Thompson,
\newblock ``On the likelihood that one unknown probability exceeds another in
  view of the evidence of two samples,''
\newblock {\em Biometrika}, vol. 24, no. 3--4, pp. 285--294, 1933.

\bibitem{AgrawalG}
Shipra Agrawal and Navin Goyal,
\newblock ``Analysis of {T}hompson sampling for the multi-armed bandit
  problem.,''
\newblock {\em Journal of Machine Learning Research - Proceedings Track}, vol.
  23, pp. 39.1--39.26, 2012.

\bibitem{OsbVanRoy14:eluderRL}
Ian Osband and Benjamin~V. Roy,
\newblock ``Model-based reinforcement learning and the eluder dimension,''
\newblock in {\em Advances in Neural Information Processing Systems 27}, 2014,
  pp. 1466--1474.

\bibitem{Gopalan15}
Aditya Gopalan and Shie Mannor,
\newblock ``Thompson sampling for learning parameterized markov decision
  processes,''
\newblock in {\em Conf. on Learning Theory, {COLT} 2015, Paris, France, July
  3-6, 2015}, 2015, pp. 861--898.

\bibitem{AbbPalCsa11:linbandits}
Yasin Abbasi-Yadkori, David Pal, and Csaba Szepesvari,
\newblock ``{Improved Algorithms for Linear Stochastic Bandits},''
\newblock in {\em Proc. NIPS}, 2011, pp. 2312--2320.

\bibitem{BookCBL}
N.~Cesa-Bianchi and G.~Lugosi,
\newblock {\em Prediction, Learning, and Games},
\newblock Cambridge University Press, 2006.

\bibitem{JakschOA10}
Thomas Jaksch, Ronald Ortner, and Peter Auer,
\newblock ``{Near-optimal Regret Bounds for Reinforcement Learning},''
\newblock {\em JMLR}, vol. 11, pp. 1563--1600, 2010.

\bibitem{GopManMan14:thompson}
Aditya Gopalan, Shie Mannor, and Yishay Mansour,
\newblock ``{Thompson Sampling for Complex Online Problems},''
\newblock in {\em Proc. International Conf. on Machine Learning}, 2014.

\bibitem{AuerCF02}
Peter Auer, Nicol{\`o} Cesa-Bianchi, and Paul Fischer,
\newblock ``Finite-time analysis of the multiarmed bandit problem,''
\newblock {\em Machine Learning}, vol. 47, no. 2-3, pp. 235--256, 2002.

\end{thebibliography}

\appendix

\subsection{Proof of Theorem \ref{thm:regret-bound-Th}}
\label{proof:regret-bound-Th}
We sketch how the proof of the result can be adapted from that of
\cite[Theorem $1$]{GopManMan14:thompson}; due to space constraints the
reader is referred to \cite{GopManMan14:thompson,Gopalan15} for
precise estimates and details. %
We first define the decision regions based on KL-divergence for each
policy $k.$ Let $ S_k := \left\{ (q,\rho) \in \mathcal{X} :
  k_{opt}{(q,\rho)} = k \right\} $ be the collection of all models
$(q,\rho) \in \mathcal{X}$ for which the optimal policy is $k.$ Denote
$k^* := k_{\text{opt}}(q^*, \rho^*)$, let $\epsilon >0,$ and define
the following sub-decision regions $\forall k \neq k^*$: {\small{
\begin{eqnarray*} 
S_k^{'} := S_k^{'}(\epsilon) =  \left\{ (q,\rho) \in S_k: D\left(f(q^*,\rho^*,k^*)||f(q,\rho,k^*)\right) \leq \epsilon \right\} \\
S_k^{''} := S_k \setminus S_k^{'} = \left\{ (q,\rho) \in S_k:D\left(f(q^*,\rho^*,k^*)||f(q,\rho,k^*)\right) > \epsilon \right\}.
\end{eqnarray*}
}} Let $N_k(l) = \sum_{i=1}^{l} 1_{\{ (q_i,\rho_i) \in S_k \}}$ be the
number of times up to and including epoch $l$ for which the policy
employed by Algorithm~\ref{algo:TS} is $k.$ Also, $N_k(l) =
\sum_{i=1}^{l} 1_{\{ (q_i,\rho_i) \in S_k^{'}\}} + \sum_{i=1}^{l}
1_{\{(q_i,\rho_i) \in S_k^{''}\}}$.  Define $N_k^{'}(l) :=
\sum_{i=1}^{l} 1_{\{ (q_i,\rho_i) \in S_k^{'}\}}$ and $N_k^{''}(l) :=
\sum_{i=1}^{l} 1_{\{(q_i,\rho_i) \in S_k^{''}\}}.$ Next, it can be
shown that the posterior on $S_k^{''}$ decays exponentially with $t$,
leading to a negligible, i.e., $O(1)$, regret from $S_k^{''}.$
To obtain posterior distribution on $ S_k^{'}$ to be small, we need
$\sum_{k}^{k_{\max}} N_k(l) D\left(f(q^*,\rho^*,k) || f(q,\rho,k)
\right) \approx \log L. $ This mean that suboptimal models are sampled
as long as their posterior probability mass is greater than
$\frac{1}{L}.$
If the posterior probability of a model parameter is less than
$\frac{1}{L},$ then the number of times that parameter sampled up to
epoch $L$ is $O(1)$ and this is negligible compared to regret. It is
thus enough to bound the maximum amount of time that the posterior
probability of any $S_k^{'}$, $k \neq k^*$, can stay above $1/L$, when
non-trivial regret is incurred. We now define $N^{'}(l) := \left(
  N_k^{'}(l) \right),$ at $l \geq 0,$ $N^{'}(0)= (0,\cdots,0).$ The
policy $k$ is eliminated when all its model losses exceed $\log L$.
Let $\tau_1$ be the first time when some policy $k_1$ is eliminated,
$k_1 \neq k^*.$
The play count of policy $k_1$ fixed at $N_{k_1}^{'}(\tau_1)$ for
remaining horizon up to $L.$ Next $\tau_2 \geq \tau_1$ when policy
$k_2 \notin \{k^*,k_1 \}$ is eliminated and play count of $k_2$ fixed
at $N_{k_2}^{'}(\tau_2).$ This process goes on until all suboptimal
policies eliminated.  $N^{'}(\tau_i) = \left( N_{k}^{'}(\tau_i)
\right)_{\{k =1 , \cdots k_{\max} \}}$ is play count vector of all
policies at time $\tau_i.$ Let $Y_i := N^{'}(\tau_i) = \left(
  N_{k}^{'}(\tau_i) \right)_{\{k =1 , \cdots k_{\max} \}}.$ Since the
play count of policy $k_i$ fixed at $N_{k_i}^{'}(\tau_i)$ for
remaining horizon, we have constraints $Y_i(k_j) = Y_j(k_j),$ for $i
\geq j.$ That means plays of policy $k_j$ do not occur after time
$\tau_j.$
Let $D(f(q,\rho)) := \left( D\left(f(q^*,\rho^*,k)||f(q,\rho,k)\right)
\right)_{ \{k =1 , \cdots k_{\max} \}}$ is a vector of the marginal
Kullback-Leibler divergences for all policies.  As the policy $k_i$
eliminated at time $\tau_i,$ this translates into the following
problem: $\min_{(q,\rho)\in S_{k_i}^{'}} \ip{Y_i}{D(f(q,\rho))} \geq
\frac{1+\epsilon}{1-\epsilon}\log L$, where $\ip{x}{y}$ denotes the
standard inner product in Euclidean space. We summarize the discussion
on the elimination of suboptimal policies in the following constrained
optimization problem that depends on the marginal KL divergences.
{\small{
\begin{equation}
  \label{eq:objective}
  \tag{P1}
  \begin{array}{rll}
  & C(\log L):= \\
     & \max \displaystyle  \sum_{i=1}^{{k}_{\max}-1} Y_i(k_i)
    &\\
    \textrm{s.t.} & \displaystyle Y_i \in \mathbb{R}_{+}^{k_{\max}}, \ \ i=1,\cdots, k_{\max}-1 \\
   & \displaystyle   Y_i(k_{\max} ) = 0, \ \   k=1,\cdots, k_{\max}-1 \\
   & \displaystyle Y_i \geq Y_j, \ \ i \geq  j, \ j =1 \cdots, k_{\max}-1 \\
   & \displaystyle Y_i(j) = Y_j(j), \ \ i \geq j, \ j= 1, \cdots, k_{\max}-1 \\
   & \displaystyle \sigma: \{ 1, \cdots, k_{\max}-1\} \rightarrow  \{ 1, \cdots, k_{\max}\} -\{k^*\} \text{ injective } \\
   & \displaystyle \min_{(q,\rho) \in S^{'}_{\sigma(i)}} \ip{ Y_i}{ D(f(q,\rho))} 
  \displaystyle = \frac{1+\epsilon}{1-\epsilon} \log L, \ \\
   &  i=1, \cdots, k_{\max} -1.
    \end{array}
\end{equation}
}}
\vspace{-\baselineskip}
\subsection{Preliminary Results Towards Proving Theorem \ref{thm:reduction}}
We collect here some useful assertions towards showing the result.  
We first note that the variation distance provides a lower bound on
KL-divergence and it is given as
{\small{
\begin{eqnarray}
D(f(q^*,\rho^*,k)||f(q,\rho,k)) &\geq & \frac{1}{2 \ln 2} d^2(f(q^*,\rho^*,k), f(q,\rho,k)) \nonumber \\
&\geq & \frac{1}{\ln 2} d_k(q,\rho)
\label{eq:KL-var-dist}
\end{eqnarray}
}}
Here, $d(f(q^*,\rho^*,k), f(q,\rho,k))$ is variation distance between $f(q^*,\rho^*,k)$ and $f(q,\rho,k)$ and this is described as follows.
\begin{eqnarray*}
d(f(q^*,\rho^*,k), f(q,\rho,k)) =
2 \big \vert f(q^*,\rho^*,k) - f(q,\rho,k) \big \vert 
\end{eqnarray*}
We can rewrite ${\big \vert f(q^*,\rho^*,k) - f(q,\rho,k) \big \vert}^2$  as follows.
\begin{eqnarray*}
{\big \vert f(q^*,\rho^*,k) - f(q,\rho,k) \big \vert}^2 =
\left[ \overline{q}^k (\rho -1) - \overline{q^*}^k(\rho^*-1) \right]^2,
\end{eqnarray*}
where $\overline{q} = 1- q,$ and $\overline{q^*} = 1- q^*.$
Define  
$
d_k(q,\rho) := \left[ \overline{q}^k (\rho -1) - \overline{q^*}^k(\rho^*-1) \right]^2,
$
and $d(q,\rho) := \left[ d_1(q,\rho),\cdots, d_{k_{\max}}(q,\rho)\right].$ We need the following series of lemmas to 
prove Theorem \ref{thm:reduction}.
\begin{lemma}
For every $\epsilon > 0,$ there exists $\delta>0$ such that if $(q, \rho)$ and $(q^*,\rho^*)$ sufficiently away  and 
$
D(f(q^*,\rho^*,k^*) || f(q,\rho,k^*)) \leq \epsilon
$
then
$
D(f(q^*,\rho^*,k) || f(q,,\rho, k)) \geq \delta.
$
\label{lemma:kl-lower-bound}
\end{lemma}
\begin{proof}
Since $(q,\rho)$ is sufficiently away from $(q^*,\rho^*),$ 
the difference $|\overline{q}^k - \overline{q^*}^k|$ will be positive for all policies $k=1,2,\cdots, k_{\max}.$
We set  
\begin{eqnarray*}
\delta_1 := \min_{k}\left[ \overline{q}^k (\rho -1) - \overline{q^*}^k(\rho^*-1) \right]^2 > 0.
\end{eqnarray*}
Then, using inequality in \eqref{eq:KL-var-dist}, we obtain
\begin{eqnarray*}
D(f(q^*,\rho^*,k)||f(q,\rho,k)) &> & \frac{1}{\ln 2} \delta_1  > 
 \delta,
\end{eqnarray*}
where $\delta =  \frac{1}{\ln 2} \delta_1 .$
This completes the proof.
\end{proof}
\begin{lemma}
One can find $\epsilon >0$ such that it can not happen that there exists $(q,\rho) \notin \mathcal{N}_{\epsilon_1}(q^*,\rho^*)$
and $k, k^{'},$ $k \neq k^{'}$ for which 
$ d_k(q,\rho) \leq \epsilon$  and $d_{k^{'}}(q,\rho) \leq \epsilon.$
\label{lemma:comp-var-dist}
\end{lemma}
\begin{proof}
Suppose $d_k(q,\rho) = d_{k^{'}}(q,\rho) =0,$ then
we will have
\begin{eqnarray}
\frac{\overline{q}^k}{\overline{q^*}^k} &=& \frac{1-\rho^*}{1-\rho}  = \frac{\overline{q}^{k^{'}}}{\overline{q^*}^{k^{'}}}
\label{eq:q-ineq}
\end{eqnarray}
This implies that 
$
\overline{q}^{k-k^{'}} = \overline{q^*}^{k-k^{'}}
$
Now observe that when $k \neq k^{'}$ and $q$ is not in neighborhood of $q^*,$ 
so equality \eqref{eq:q-ineq} is not true. 
This means that our assumption $d_k(q,\rho) = d_{k^{'}}(q,\rho) =0$ is not true.
Further, it implies that only one of the following claim is true. 
\begin{enumerate}
\item if $d_k(q,\rho) = 0, $ then $d_{k^{'}}(q,\rho) > 0$ for $k \neq k^{'}$
\item if   $d_{k^{'}}(q,\rho) = 0$ then $d_k(q,\rho) > 0 $ for $k \neq k^{'}$.
\end{enumerate}
In other word, we can find $\epsilon >0 $ for which either $d_k(q,\rho) \leq \epsilon, $  $d_{k^{'}}(q,\rho) > \epsilon $
is true,  or $d_{k^{'}}(q,\rho) \leq \epsilon $  $d_k(q,\rho) > \epsilon$ is true. 
\end{proof}
\begin{lemma}
Consider any parameter $(q,\rho) \neq (q^*,\rho^*)$ and 
$(q,\rho) \notin \mathcal{N}_{\epsilon_1}(q^*,\rho^*),$ 
where $\mathcal{N}_{\epsilon_1}(q^*,\rho^*)$ is $\epsilon_1$ neighborhood of $(q^*,\rho^*).$
Then 
there exists an integer $\kappa \in \{ 1,2,3,\cdots,k_{\max}-1\}$ and 
$\Delta >0$ such that for all $(q,\rho) \notin \mathcal{N}_{\epsilon_1}(q^*,\rho^*):$
\begin{equation} 
\big\vert \{ k : d_k (q,\rho) \geq \Delta  \} \big\vert \geq \kappa.
\end{equation}
Also, for sufficiently small $\Delta >0,$ we have $\kappa = k_{\max}-1.$ 
\label{lemma:size-set-bound}
\end{lemma}
\begin{proof}
From Lemma \ref{lemma:comp-var-dist}, notice that in $d(q,\rho),$ there can be at most  one element which can be zero or 
arbitrary close zero, say, $d_{l}(q,\rho) \leq \epsilon$ and  remaining entries, $d_k(q,\rho) > \epsilon,$  $k \neq l.$
When $(q,\rho) \notin \mathcal{N}_{\epsilon_1}(q^*,\rho^*),$ implies $|q - q^*| \geq \epsilon_1$ and 
$|\rho-\rho^*| \geq \epsilon_1.$
Thus we obtain 
{\small{
\begin{eqnarray*}
d_k(q,\rho) &=& \left[ \overline{q}^k (\rho -1) - \overline{q^*}^k(\rho^*-1) \right]^2 \nonumber \\
& \geq& \min_{k} \left[ (\overline{q^*}+\epsilon_1)^k ((\rho^*+\epsilon_1) -1) - \overline{q^*}^k(\rho^*-1) \right]^2 
\end{eqnarray*}
 $\Delta := \min_{1\leq k \leq k_{\max}} \left[ (\overline{q^*}+\epsilon_1)^k ((\rho^*+\epsilon_1) -1) - \overline{q^*}^k(\rho^*-1) \right]^2, $ this $\Delta > 0.$
 }}
Now, combining Lemma \ref{lemma:comp-var-dist}, and $d_k(q,\rho) \geq \Delta,$ there exists $\kappa \in \{1, 2,3, \cdots,k_{\max}-1\},$
 such that for $(q,\rho) \notin \mathcal{N}_{\epsilon_1}(q^*,\rho^*),$ we have
\begin{eqnarray}
\big\vert \{k: d_k(q,\rho) \geq \Delta \}\big\vert \geq \kappa.
\end{eqnarray} 
Further, for sufficiently small $\epsilon >0,$ and 
$(q,\rho) \notin \mathcal{N}_{\epsilon_1}(q^*,\rho^*),$ we have $k_{\max}-1$ nonzero entries in vector 
$d(q,\rho).$ In this case, fix $\Delta = \epsilon,$  we obtain 
$\big\vert \{k: d_k(q,\rho) > \Delta \} \big\vert = k_{\max}-1.$ Thus,
$\kappa = k_{\max}-1.$
\end{proof}

\subsection{Proof of Theorem \ref{thm:reduction}}
\label{proof:Thm-parameter-depend-regret}
From lemma \ref{lemma:size-set-bound}, we know that for sufficiently small $\Delta > 0,$ we have $\kappa = k_{\max}-1$ and 
 $\big\vert \{ k : d_k (q,\rho) \geq \Delta  \} \big\vert = k_{\max}-1.$	
Thus, we have $\Delta_2 = \frac{1}{\ln 2} \Delta >0$ and $\kappa =k_{\max}-1$ such that 
{\small{
\begin{eqnarray*}
\big\vert \{k \in \mathcal{K}: k \neq k^*, D(f(q^*,\rho^*,k)||f(q,\rho,k)) \geq \Delta_2 \} \big\vert 
= k_{\max} -1
\end{eqnarray*}
}}
From Lemmas \ref{lemma:kl-lower-bound}, \ref{lemma:size-set-bound} and eqn. \eqref{eq:KL-var-dist},
we note that all assumptions in \cite[Proposition $2$]{GopManMan14:thompson} are satisfied.
Note that the upper bound on $C(\log L) $ is given in \cite[Proposition $2$]{GopManMan14:thompson} and it is as follows.
\begin{eqnarray*}
C(\log L) \leq \left( \frac{k_{\max}-\kappa}{\Delta_2} \right) \frac{2(1+\epsilon)}{1-\epsilon} \log L
\end{eqnarray*}
Here, we substitute $\kappa = k_{\max}-1,$ and required upper bound on $C(\log L)$ follows.

\end{document}